\documentclass[letterpaper]{article} 
\usepackage{aaai24}  
\usepackage{times}  
\usepackage{helvet}  
\usepackage{courier}  
\usepackage[hyphens]{url}  
\usepackage{graphicx} 
\usepackage{natbib}  
\usepackage{caption} 
\usepackage{amsmath}
\usepackage{amssymb}
\usepackage{amsthm}
 
\newtheorem{proposition}{Proposition}

\newtheorem{theorem}{Theorem}
\newtheorem{lemma}{Lemma}
\newtheorem{remark}{Remark}
\newtheorem{definition}{Definition}
\usepackage{bbding}

\newcommand{\eg}{{\emph{e.g.}}}

\usepackage{newfloat}
\usepackage{listings}
\usepackage{multirow}
\usepackage{array} 

\usepackage{algorithm}
\usepackage{algorithmic}
\usepackage{newfloat}
\usepackage{listings}
\DeclareCaptionStyle{ruled}{labelfont=normalfont,labelsep=colon,strut=off} 
\floatstyle{ruled}
\newfloat{listing}{tb}{lst}{}
\floatname{listing}{Listing}
\title{Exploring Diverse Representations for Open Set Recognition}
\author{
    Yu Wang,
    Junxian Mu,
    Pengfei Zhu\thanks{Corresponding author},
    Qinghua Hu
}
\affiliations{
    \textsuperscript{\rm 1}College of Intelligence and Computing, Tianjin University\\
    wang.yu@tju.edu.cn, jxmu@tju.edu.cn, zhupengfei@tju.edu.cn, huqinghua@tju.edu.cn
}

\usepackage{bibentry}
\begin{document}
\maketitle
\begin{abstract}
Open set recognition (OSR) requires the model to classify samples that belong to closed sets while rejecting unknown samples during test. Currently, generative models often perform better than discriminative models in OSR, but recent studies show that generative models may be computationally infeasible or unstable on complex tasks. In this paper, we provide insights into OSR and find that learning supplementary representations can theoretically reduce the open space risk. Based on the analysis, we propose a new model, namely Multi-Expert Diverse Attention Fusion (MEDAF), that learns diverse representations in a discriminative way. MEDAF consists of multiple experts that are learned with an attention diversity regularization term to ensure the attention maps are mutually different. The logits learned by each expert are adaptively fused and used to identify the unknowns through the score function. We show that the differences in attention maps can lead to diverse representations so that the fused representations can well handle the open space. Extensive experiments are conducted on standard and OSR large-scale benchmarks. Results show that the proposed discriminative method can outperform existing generative models by up to 9.5\% on AUROC and achieve new state-of-the-art performance with little computational cost. Our method can also seamlessly integrate existing classification models. Code is available at https://github.com/Vanixxz/MEDAF.
\end{abstract}

\section{Introduction}
Humans can realize what they have not learned, which provides a valuable basis for actively formulating questions and seeking information~\cite{markman1979realizing}. But this is quite challenging for artificial intelligence models: they will mistakenly identify samples of an unknown class as a known class, thereby preventing them from having the ability to think. To this end, open set recognition (OSR) is proposed to enable the models to recognize unknown samples while correctly identifying known samples~\cite{scheirer2013open}.

Preliminary attempts on OSR have been made, in which existing methods can be categorized as discriminative methods and generative methods. Discriminative methods transform OSR into a classical discriminative task by imitating the open space using available data or features, such as explicitly estimating the probability of a test unknown sample~\cite{bendale2016towards}, assigning placeholders to predict the distribution of unknown data~\cite{zhou2021learning}, and utilizing supervised contrastive learning to improve the quality of representations~\cite{xu2023contrastive}. These methods are simple and robust in various scenarios and easy to scale to large-scale tasks, but their performance is often limited due to the difficulty of modeling the open space. By contrast, generative methods encode known classes into various latent distributions using generative models and identify unknowns by measuring their distance to known distributions, such as mapping the latent known classes to prototypes~\cite{yang2022convolutional}, Gaussian distributions, and manifolds~\cite{oza2019c2ae}. These methods have shown promising results and achieved state-of-the-art performance.

Despite effectiveness, generative methods are shown to be computationally infeasible or performatively unstable on complex tasks. \cite{xu2023contrastive} pointed out that training generative models significantly increases the total training cost of the recognition system, and \cite{vaze2022openset} found it infeasible to train generative methods on large-scale benchmarks, \emph{e.g.,} ImageNet. Moreover, \cite{guo2021conditional} and \cite{huang2023class} empirically proved that generative methods suffer from significant performance degradation in recognizing unknowns when unknown classes are similar to known classes.

The aforementioned analysis raises a straightforward question: Is it possible to find the rationale of existing methods to design OSR models that are effective and robust as well as efficient and scalable?

In this paper, we first provide insights into OSR and theoretically find a key factor of generative methods is learning diverse representations that contain additional information to discriminative representations for identifying known classes. Subsequently, we propose a simple discriminative method that learns diverse representations using only vanilla discrimination models. Inspired by the mixture-of-experts model, we design multiple experts that share shallow layers and own expert-independent layers. To collaboratively learn diverse representations, the attention maps learned by the experts are constrained to be different using a diversity regularization term. Finally, the logits of each expert are adaptively fused by a gating network. The fused logits are used to classify knowns after the Softmax function as well as identifying unknowns in the score function. We show that the differences in attention maps can lead to diverse representations and the fused representations can reduce the open space risk, thereby boosting the ability to identify unknowns.

Extensive experiments on standard and large-scale OSR benchmarks demonstrate that the proposed method outperforms existing discriminative and generative models by up to 9.5\% on AUROC and achieves new state-of-the-art performance. MEDAF is quite simple to implement and can seamlessly integrate existing classification models. The contributions of this study can be summarized as follows:
\begin{itemize}
\item We provide theoretical insights into OSR and find that learning diverse representations is the key factor in reducing the open space risk. 
\item We propose a simple yet effective method that learns multiple experts by constraining the learned attention map of each expert to be mutually different and fuses diverse representations to identify knowns and unknowns.
\item Experiments on standard and large-scale benchmarks demonstrate that the proposed method shows clear advantages over baselines with little computational cost.
\end{itemize}

\section{Related Works}

\subsection{Discriminative OSR Methods}
Since \cite{scheirer2013open} provided a mathematical definition for OSR, \cite{hendrycks2016a} proposed a deep learning method that negates maximum softmax probability to reject unknowns. \cite{bendale2016towards} demonstrated the instability of Softmax probabilities and designed an OpenMax function as a replacement for Softmax. \cite{zhou2021learning} treated placeholders as representative points to address overconfidence in predictions for unknown classes. \cite{xu2023contrastive} utilized supervised contrastive learning to enhance the efficacy of representations. 

Discriminative methods are efficient and scalable, but their performance is often inferior to generative methods.

\subsection{Generative OSR Methods}
\textbf{Generation for known classes.}
\cite{oza2019c2ae} proposed a two-stage approach, which trains an encoder for closed-set recognition and then adds class-conditional information to train a decoder for unknown detection. \cite{chen2020learning} developed RPL to identify if a sample is known or unknown based on its deviation from reciprocal points. \cite{sun2020open} used a variational autoencoder (VAE) for constraining different latent space features to fit different Gaussian models for unknown detection. \cite{guo2021conditional} combined VAE and capsule network to fit a predefined distribution, promoting consistency of features within the class. \cite{yang2022convolutional} proposed GCPL, which replaces the Softmax classifier with an open-world-oriented prototype model.

\noindent\textbf{Generation for unknown information.}
\cite{neal2018open} enriched the training dataset by generating samples similar to training examples but not belonging to any specific classes with GANs \cite{ian2014gan}. Based on RPL, \cite{chen2022adversarial} proposed ARPL, which enhances the ability of the model to differentiate by generating misleading samples. \cite{kong2022opengan} expanded the training set to improve the recognition of unknown samples using generated samples and auxiliary datasets. \cite{moon2022difficulty} considered generating samples of different difficulty levels to improve the robustness of the model. \cite{liu2023opentext} proposed a label-to-prototype mapping function to construct prototypes for both known classes and unknown classes.

Generative methods presume and learn different distributions for known classes or synthesize pseudo-unknowns around the known samples. Recent works show outstanding performance but are shown to be unreliable in complex tasks, e.g., the unknowns being similar to the knowns, and computationally infeasible in large-scale datasets.

\section{Delving Deep into Open Set Recognition}
\cite{scheirer2013open} aimed to find the optimal solution in the hypothesis space for OSR tasks by minimizing both empirical and open space risks. Since obtaining precise information about unknown samples during training is impossible, we analyze the potential risks on the test set to measure the performance of model. 

Misclassifying samples in the known test set $\mathcal{D}_T^K$ increases closed-set risk $\mathcal{R}_c$ while accepting from the unknown set $\mathcal{D}_T^U$ induces open space risk $\mathcal{R}_o$. To measure the performance of model $f(\theta)$ with training set $\mathcal{D}_K=\{(\boldsymbol{x}_i, y_i)\}^{N_K}_{i=1},\ y_i\in\{1,...,K\}$, we denote the test set as $\mathcal{D}_T=\mathcal{D}_T^K\cup\mathcal{D}_T^U$, the label of unknown samples as $U$, and the proportion of unknown samples as $\alpha$, the potential risk $\mathcal{R}_T\big(\mathcal{D}_T,f(\theta)\big)$ can be formulated as
\begin{equation}
\begin{aligned}
    \mathcal{R}_T=(1-\alpha)\cdot\mathcal{R}_c\big(\mathcal{D}^K_T,f(\theta)\big)+\alpha\cdot\mathcal{R}_o\big(\mathcal{D}^U_T,f(\theta)\big).  
\label{eq:base_target}
\end{aligned}
\end{equation}

Closed-set risk $\mathcal{R}_c$ can be approximated by the empirical risk $\mathcal{R}_\epsilon$ using $\mathcal{D}_K$. We can transform $\mathcal{R}_c$ into $\mathcal{R}_\epsilon$ during training. Mathematically, $\mathcal{R}_\epsilon$ can be formulated as Eq. \eqref{eq:emp_risk} considering the cross-entropy loss function:

\begin{equation}
\begin{aligned}
\mathcal{R}_c\big(\mathcal{D}^K_T,f(\theta)\big)&=\mathcal{R}_\epsilon\big(\mathcal{D}_K,f(\theta)\big)\\&=\frac1{N_K}\sum^{N_K}_{i=1}\mathcal{L}_{ce}(\boldsymbol{p}_i,\hat{\boldsymbol{p}}_i)\\&=\frac1{N_K}\sum^{N_K}_{i=1}-{\rm log}\,p(\hat{y}_i=y_i|\boldsymbol{z}_i),
\label{eq:emp_risk}
\end{aligned}
\end{equation}
where $\boldsymbol{z}_i$ denotes the global representation.

A common way to reject an unknown sample $(\boldsymbol{x}_u,y_u)\in\mathcal{D}_T^U$ is to measure the probability that the model predicts it belongs to any known class $k$ is higher than threshold $\tau$: 

\begin{equation}
p(\hat{y}_u=U|\boldsymbol{z}_u)={\rm max}_{k\in K}p(\hat{y}'_u=k|\boldsymbol{z}_u),
\label{eq:prob_unknown}
\end{equation}

\begin{equation}
\hat{y}_u=\left\{\begin{aligned}  
        &k\ \ \quad{\rm{if}}\ p(\hat{y}_u=U|\boldsymbol{z})\geq\tau\\  
        &U\,\,\quad{\rm{if}}\ p(\hat{y}_u=U|\boldsymbol{z})<\tau  
    \end{aligned}\right..
\label{eq:rej_unknown}
\end{equation}
Accordingly, the sample $(\boldsymbol{x}_u,y_u)$ will be wrongly accepted if $p(\hat{y}_u=U|\boldsymbol{z}_u)\leq\tau$.
Obviously, the open space risk of wrongly accepting an unknown sample $(\boldsymbol{x}_u,y_u)$ will be reduced if the $(\boldsymbol{x}_u,y_u)$'s probabilities of belonging to all the known classes decrease, making
$p(\hat{y}_u=U|\boldsymbol{z})$ increased. Therefore, Eq. \eqref{eq:base_target} can be formulated as follows:

\begin{proposition}  
With the test set $\mathcal{D}_T$ and the model $f(\theta)$, as $p(\hat{y}=y|\boldsymbol{z})$ increases, the potential risk $\mathcal{R}_T$ decreases.
\end{proposition} 

\begin{equation}
\begin{aligned}
\mathcal{R}_T=&(1-\alpha)\cdot\overbrace{\Big(\frac1{N_K}\sum^{N_K}_{i=1}-{\rm log}\,p(\hat{y}_i=y_i|\boldsymbol{z}_i)\Big)}^{Closed-set\ risk\ \mathcal{R}_c}\\&+\alpha\cdot\underbrace{\Big(1-\frac{1}{N_u}\sum_{u=1}^{N_u}p(\hat{y}_u=U|\boldsymbol{z}_u)\Big)}_{Open\ space\ risk\ \mathcal{R}_o}.
\label{eq:potential_risk}
\end{aligned}
\end{equation}

According to Eq. \eqref{eq:potential_risk}, mitigating potential risks $\mathcal{R}_T$ requires enhancing $p(\hat{y}=y|\boldsymbol{z})$.
\begin{lemma} \cite{cover2012elements} 
For a pair of random variables $(X,Y)\sim p(x,y)$, the conditional entropy $H(Y|X)$ given $X=x$ is:
\begin{equation}
H(Y|X=x)=-\sum_{y\in\mathcal{Y}}p(y|x){\rm log}\ p(y|x).
\label{eq:cond_entro}
\end{equation}
\end{lemma}
During the process of training models using cross-entropy, minimizing the conditional entropy between the true labels and the learned representations can optimize the value of cross-entropy. For binary classification between known and unknown, $p(\hat{y}=y|\boldsymbol{z})$ is negatively correlated with $H(y|\boldsymbol{z})$. Thus, we can reduce open space risk by decreasing $H(y|\boldsymbol{z})$, so the quality of given $\boldsymbol{z}$ affects the model's performance. However, DNNs often use a limited set of highly distinguishable features for classification, which is called ``shortcut learning" \cite{geirhos2020shortcut}. This may lead to confusion between unknowns and knowns when they share similar features learned by DNNs. Several researchers tried to enable the model to form more comprehensive attention,~\cite{wang2018deep} proposed to leverage multi-scale features. We mathematically define two kinds of representations and analyze the impact of them on the OSR task. 

\begin{definition}
\textsc{Basic Discriminative Representations} is the representations $\boldsymbol{M}_d$ captured by the baseline model $g(\psi)$ and further be utilized to obtain global representation $\boldsymbol{z}'$, which is formalized by:
\begin{equation}
g(\boldsymbol{x},\psi)=\boldsymbol{z}' \wedge\ I(y;\boldsymbol{z}')\approx I(y;\boldsymbol{M}_d).
\label{eq:defin_fd}
\end{equation}
\end{definition}

\begin{definition}
\textsc{Supplementary Discriminative Representations} is the representations $\boldsymbol{M}_s$ of class-relevant regions $r_s$ aside from the focal area $r_d$ is formalized by:
\begin{equation}
I(y;\boldsymbol{M}_s)>0\ \wedge\ r_s\cap r_d=\phi,
\label{eq:defin_fs}
\end{equation}
where $\boldsymbol{M}$ denotes the feature map obtained with CNNs and $\phi$ indicates there is no overlap between $r_s$ and $r_d$, while the representations of the focal area are denoted as $\boldsymbol{M}_d$. 
\end{definition}

Based on the above definitions, we have the following theorem.
\begin{theorem}
    With $(\boldsymbol{x},y)$ and its global representation $\boldsymbol{z}$, the conditional entropy of $y$ and $\boldsymbol{z}$ with $\boldsymbol{M}_d$ and additionally utilizing $\boldsymbol{M}_s$ is expressed as $H(y|\boldsymbol{z}'),H(y|\boldsymbol{z})$, satisfying $H(y|\boldsymbol{z}')-H(y|\boldsymbol{z})=\zeta>0$.
\label{thm:1}
\end{theorem}

\begin{proof}
According to Eq. \eqref{eq:cond_entropy}, as $H(y)$ is a constant, the difference between conditional entropy $H(y|\boldsymbol{z})$ and $H(y|\boldsymbol{z}')$ is determined by the difference of mutual information term $I(y;\boldsymbol{z})$ and $I(y;\boldsymbol{z}')$.
\begin{equation}
\begin{aligned}
H(y|\boldsymbol{z}')-H(y|\boldsymbol{z})&=H(y)-I(y;\boldsymbol{z}')-H(y)+I(y;\boldsymbol{z})
\\&=I(y;\boldsymbol{z})-I(y;\boldsymbol{z}').
\label{eq:cond_entropy}
\end{aligned}
\end{equation}

\begin{figure*}[t]
  \begin{center}
  \includegraphics[width=0.8\linewidth]{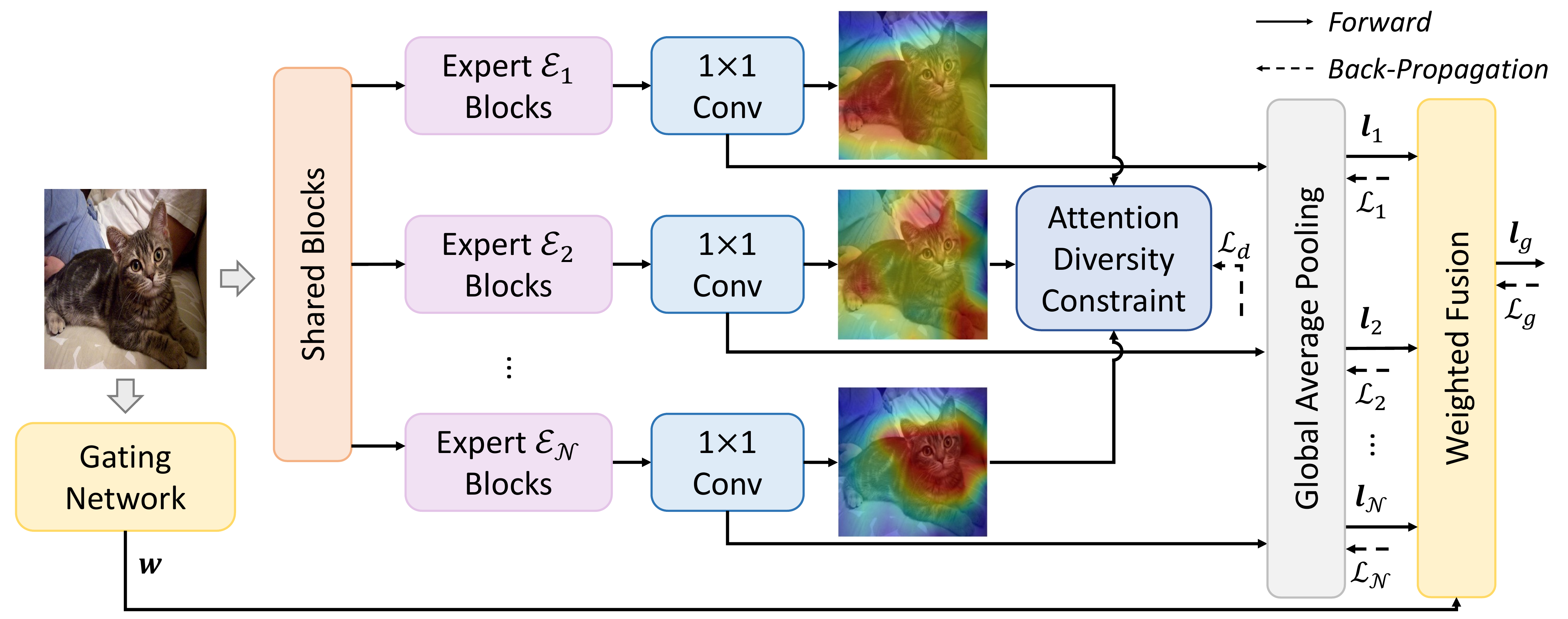}
    \caption{Illustration of the proposed MEDAF method. MEDAF consists of a multi-expert feature extractor to explore diverse representations by constraining the learned attention map of each expert to be mutually different. Then a gating network adaptively generates weights to integrate expert-independent predictions.}
    \label{fig:Fig_arch}
    \end{center}
\end{figure*}

Since $\boldsymbol{z}'$ only considering $\boldsymbol{M}_d$, $I(y;\boldsymbol{z}')$ is equivalent to $I(y;\boldsymbol{M}_d)$, while $I(y;\boldsymbol{z})$ is equivalent to $I(y;\boldsymbol{M}_d,\boldsymbol{M}_s)$, and the difference is
\begin{equation}
\begin{aligned}
I(y;\boldsymbol{z})&-I(y;\boldsymbol{z}')=I(y;\boldsymbol{M}_d,\boldsymbol{M}_s)-I(y;\boldsymbol{M}_d)\\&=I(y;\boldsymbol{M}_d)+I(y;\boldsymbol{M}_s|\boldsymbol{M}_d)-I(y;\boldsymbol{M}_d)
\\&=I(y;\boldsymbol{M}_s|\boldsymbol{M}_d).
\label{eq:mutual_info}
\end{aligned}
\end{equation}

$I(y;\boldsymbol{M}_s|\boldsymbol{M}_d)$ represents the reduction in conditional entropy from the knowledge of $\boldsymbol{M}_s$ when $\boldsymbol{M}_d$ is given, and it is significant to discuss its value.
\begin{lemma} \cite{cover2012elements} 
For any three random variables $X,Y$ and $Z$,
\begin{equation}
I(X;Y|Z)\geq0,
\label{eq:info_geq}
\end{equation}
with equality if and only if $X$ and $Y$ are conditionally independent given $Z$.
\end{lemma}

As $\boldsymbol{M}_s$ is related with class $y$ and $\boldsymbol{M}_d,\ \boldsymbol{M}_s$ are representations of different regions, the equality condition of Eq. \eqref{eq:info_geq} does not hold, so we can derive that $I(y;\boldsymbol{M}_s|\boldsymbol{M}_d)>0$, which indicates that introducing $\boldsymbol{M}_s$ can reduce the open space risk.
\end{proof}

As the model gives lower prediction probabilities on other non-corresponding classes with $\boldsymbol{M}_s$, it can predict a low probability of unknown samples belonging to any known class, thereby constraining the open space risk that induced by considering these samples as known.

\subsubsection{Connections to existing methods.} 
The above finding can explain the working mechanisms of some existing generative methods, \emph{i.e.,} exploring diverse representations $\boldsymbol{M}_s$ in addition to $\boldsymbol{M}_d$:
\begin{itemize}
    \item Prototype or auto-encoder based methods learn $\boldsymbol{M}_s$ by mapping it to prototypical points or manifolds. But prototypes cannot provide much information of $\boldsymbol{M}_s$ and even under learn $\boldsymbol{M}_d$ due to the limited representation ability, and AEs may learn representations that are unrelated to classification~\cite{huang2023class}.
    \item Unknown sample generation methods enable $\boldsymbol{M}_s$ by learning to distinguish knowns and pseudo unknowns. However, the generated unknowns are unreliable due to limited information, thereby learning unreliable $\boldsymbol{M}_s$~\cite{kong2022opengan}.
\end{itemize}

To learn reliable $\boldsymbol{M}_s$ flexibly, we design a new model that obtains $\boldsymbol{M}_s$ through fusing different representations learned by multiple local networks to output the final prediction, which can seamlessly build upon existing classification models and further enable them exploring diverse representations. 

\section{Multi-Expert Diverse Attention Fusion}
\subsection{Multi-Expert Architecture}
Inspired by the mixture of experts proposed by \cite{shazeer2017moe}, 
we design multiple experts that share parameters in shallow layers and have individual parameters in deep layers. Each expert can specialize in capturing specific semantic features. 

As shown in Figure~\ref{fig:Fig_arch}, we set multi-experts represented as $\{\mathcal{E}_i\}_{i=1}^\mathcal{N}$. Given a training sample $(\boldsymbol{x}, y)$, it first passes through the shared shallow layer to produce an intermediate representation with more detail, then is fed to various experts, which capture different discriminative regions. 

\begin{remark}
According to the specific demand and backbone, the number of expert-independent blocks and experts can be flexibly adjusted.
\end{remark}

\subsection{Learning Mutually Diverse Representations}
The key to our method is to ensure multiple experts learn different representations. Simply constraining the learned representations to be different may lead to learning unnecessary information, so we use class activation mapping (CAM) \cite{zhou2016learning} that highlights the most critical regions in an image contributing to class $y$. 

Concretely, global averaging pooling (GAP) outputs global representation $\boldsymbol{z}$ by calculating the average value of each pixel $(h,w)$ in feature map $\boldsymbol{M}$. By replacing the \textbf{GAP-Linear} with \textbf{1$\times$1 Conv-GAP}, we can obtain the feature map $\boldsymbol{M}\in\mathbb{R}^{[K,H,W]}$ and the global representation $\boldsymbol{z}\in\mathbb{R}^{[K,1,1]}$. Given label $y$, we can obtain the corresponding feature map $\boldsymbol{M}_y$, which is equivalent to CAM. With feature $\boldsymbol{M}$ and its global representation $\boldsymbol{z}$, the GAP process can be formally described as Eq. \eqref{eq:gap}, where $\delta$ denotes values on the nonsignificant regions (\eg, background).

\begin{equation}
\begin{aligned}
\boldsymbol{z}&={\rm GAP}(\boldsymbol{M})=\frac{1}{|\boldsymbol{M}|}\sum_{(h,w)}{\boldsymbol{M}^{(h,w)}}\\&=\underbrace{\frac{\lambda_d}{|\boldsymbol{M}_d|}\sum_{(h,w)}{\boldsymbol{M}^{(h,w)}_d}+\frac{\lambda_s}{|\boldsymbol{M}_s|}\sum_{(h,w)}{\boldsymbol{M}^{(h,w)}_s}}_{Class-relevant\ representations}+\delta\\&=\lambda_d*{\rm GAP}(\boldsymbol{M}_d)+\lambda_s*{\rm GAP}(\boldsymbol{M}_s)+\delta.
\label{eq:gap}
\end{aligned}
\end{equation}

In addition to class-independent low-activation regions $\delta$, regions with high activation values (\eg, $\boldsymbol{M}_d$) in the CAM represent the location that the model focuses on, while regions with low activation values (\eg, $\boldsymbol{M}_s$) are considered as supplementary representations. By forming diversity in $\boldsymbol{M}_d$ and $\boldsymbol{M}_s$ of each pair of experts, we can constrain differences in the regions that experts focus on and combine them to develop a comprehensive representation of the classification. We zero activation values lower than the mean $\mu$ to filter useless regions. Let $\boldsymbol{M}\mathrm{'}_{y}^{i}$ denotes the processed CAM on expert $\mathcal{E}_i$, and it can be calculated as
\begin{equation}
\boldsymbol{M}\mathrm{'}_{y}^{i}={\rm{ReLU}}(\boldsymbol{M}_{y}^{i}-\mu).
\label{eq:cam_relu}
\end{equation}  

We perform the cosine similarity between the processed CAM in a pairwise manner. The computed similarities are summed to compute the regularization loss $\mathcal{L}_d$ for diversified expert attention, formulated as
\begin{equation}
\mathcal{L}_d=\sum_{i=1}^{\mathcal{N}-1}\sum_{j=i+1}^{\mathcal{N}}\frac{\boldsymbol{M}\mathrm{'}_{y}^{i}\cdot\boldsymbol{M}\mathrm{'}_{y}^{j}}{\big|\big|\boldsymbol{M}\mathrm{'}_{y}^{i}\big|\big|_2*\big|\big|\boldsymbol{M}\mathrm{'}_{y}^{j}\big|\big|_2}.
\label{eq:diver_loss}
\end{equation}  

\begin{table*}[ht]
\centering
\small
\begin{tabular}{cccccc}
\hline
Method & SVHN & CIFAR10 & CIFAR+10 & CIFAR+50 & Tiny-ImageNet \\ \hline
Softmax (ICLR'16)& 0.886 & 0.677 & 0.816 & 0.805 & 0.577 \\ 
C2AE (CVPR'19)& 0.892 & 0.711 & 0.810 & 0.803 & 0.581 \\ 
CGDL (CVPR'20)& 0.896 & 0.681 & 0.794 & 0.794 & 0.653 \\ 
RPL (ECCV'20)& 0.931 & 0.784 & 0.885 & 0.881 & 0.711 \\ 
PROSER (CVPR'21)& 0.930 & 0.801 & 0.898 & 0.881 & 0.684 \\ 
CVAECap (ICCV'21)& 0.956 & 0.835 & 0.888 & 0.889 & 0.715 \\ 
ARPL (TPAMI'22)& 0.946 & 0.819 & 0.904 & 0.901 & 0.710 \\ 
NAS-OSR (AAAI'22)& 0.949 & 0.843 & 0.840 & 0.871 & - \\ 
DIAS (ECCV'22)& 0.943 & 0.850 & 0.920 & 0.916 & 0.731 \\ \hline
MEDAF & \textbf{0.957} & \textbf{0.860} & \textbf{0.960} & \textbf{0.955} & \textbf{0.800} \\ \hline
\end{tabular}
\caption{Comparison of different methods on unknown detection tasks using AUROC. All results are the average value, and the best performance values are in bold.}
\label{table_auroc}
\end{table*}

\subsection{Expert Fusion}To integrate predictions of different experts, we use a gating network to adaptively generate weights of experts. The gating network can consider expert-independent characteristics and relevance and allow fine-grained control over the contribution of each expert to the final result.

We adopt an identical feature extractor to those employed in the main network as the backbone of the gating network. Two fully connected layers are incorporated at the top of the backbone, culminating in deriving expert prediction weights. Let $\boldsymbol{l}_i\in\mathbb{R}^{K}$ denotes the logits from expert $\mathcal{E}_i$, and $\boldsymbol{w}\in\mathbb{R}^{\mathcal{N}}$ represent the weights given by gating network. The final logits $\boldsymbol{l}_g$ is calculated as follows:
\begin{equation}
\boldsymbol{l}_g=\sum_{i=1}^{\mathcal{N}}\boldsymbol{w}_i*\boldsymbol{l}_i,
\label{eq:gating_lg}
\end{equation}
where $\boldsymbol{l}_i$ denotes the logits of the $i$-th expert $\mathcal{E}_i$ while $\boldsymbol{l}_g$ denotes the integrated logits.

The overall loss function can be decomposed into global and expert-wise cross-entropy loss terms (denoted by $\mathcal{L}_{ce}^g$ and $\mathcal{L}_{ce}^i$) and the attention diversity regularization term $\mathcal{L}_d$, which can be formulated as
\begin{equation}
\mathcal{L}=\mathcal{L}_{ce}^g+\beta_1*\sum_{i=1}^{\mathcal{N}}\mathcal{L}_{ce}^i+\beta_2*\mathcal{L}_d,
\label{eq:loss_full}
\end{equation}
where $\beta_1$ and $\beta_2$ denote the scaling factors. The cross-entropy loss is formulated as
\begin{equation}
\mathcal{L}_{ce}=\sum_k-\boldsymbol{y}_k*{\rm log}\big({\rm Softmax}(\boldsymbol{l})^k\big).
\label{eq:cross_perb}
\end{equation}

\subsection{Rejecting Unknown Samples}
According to open space risk $\mathcal{R}_o$ in Eq. \eqref{eq:prob_unknown}, the probability of an unknown sample being correctly rejected is inversely proportional to the mutual information between its feature and any known class $k$. Unknown samples that share few discriminative features with known class samples are more likely to confuse the model \cite{moon2022difficulty}.

By averaging multiple effective feature maps, we obtain a more comprehensive feature map that weakens the impact of a few similar features. The value $\mathcal{S}_{ft}(\boldsymbol{x})$ is obtained through the L2 normalization of the averaged features:
\begin{equation}
\mathcal{S}_{ft}(\boldsymbol{x})=\Big|\Big|\frac{1}{\mathcal{N}}{\sum_{i=1}^\mathcal{N}\boldsymbol{M}_i}\Big|\Big|_2.
\label{eq:score_ft}
\end{equation}

To avoid degrading the original numerical difference with maximum softmax probability to reject unknowns, 
we use the maximum value of the logits $\boldsymbol{l}_g$ as a term $\mathcal{S}_{lg}(\boldsymbol{x})$ in $\mathcal{S}(\boldsymbol{x})$ and obtain the final result by linearly combining the above two terms according to Eq. \eqref{eq:score_func}, where $\gamma$ is the weights of each term. When rejecting unknown samples, we use the $\mathcal{S}(\boldsymbol{x})$ value that ensures 95\% of known samples are accepted as the threshold $\tau$. According to Eq. \eqref{eq:rej_unknown}, samples with scores lower than $\tau$ are considered unknown samples.
\begin{equation}
\mathcal{S}(\boldsymbol{x})=\mathcal{S}_{lg}(\boldsymbol{x})+\gamma*\mathcal{S}_{ft}(\boldsymbol{x}).
\label{eq:score_func}
\end{equation}

\section{Experiments}
\subsection{Implementation Details}
In experiments, we used ResNet18~\cite{he2016deep} as the backbone. In terms of optimization, we used an SGD optimizer with a momentum value of 0.9 and set the initial learning rate to 0.1 with a fixed batch size of 128 for 150 epochs. For all the compared methods, we directly used the reported results if the settings were the same or the reproduced results using their official code if the settings were different. For methods that have not released their source code, we only reported the results of experiments they have done.
\begin{table*}[t]
\centering
\small
\begin{tabular}{ccccccccc}
    \hline
    \multirow{2}[4]{*}{Method}&\multicolumn{4}{c}{In:CIFAR10 / Out:CIFAR100} & \multicolumn{4}{c}{In:CIFAR10 / Out:SVHN} \\
    \cline{2-5} \cline{6-9}
    & DTACC & AUROC & AUIN  & AUOUT & DTACC & AUROC & AUIN  & AUOUT \\ \hline
    Softmax (ICLR'16)& 0.798 & 0.863 & 0.884 & 0.825 & 0.864 & 0.906 & 0.883 & 0.936 \\ 
    RPL (ECCV'20)& 0.806 & 0.871 & 0.888 & 0.838 & 0.871 & 0.920 & 0.896 & 0.951 \\ 
    CSI (NeuIPS'20)& 0.844 & 0.916 & 0.925 & 0.900 & 0.928 & 0.979 & 0.962 & 0.990 \\ 
    GCPL (TPAMI'22)& 0.802 & 0.864 & 0.866 & 0.841 & 0.861 & 0.913 & 0.866 & 0.948 \\ 
    ARPL (TPAMI'22)& 0.834 & 0.903 & 0.911 & 0.884 & 0.916 & 0.966 & 0.948 & 0.980 \\ \hline
    MEDAF & \textbf{0.854} & \textbf{0.925} & \textbf{0.932} & \textbf{0.911} & \textbf{0.953} & \textbf{0.991} & \textbf{0.980} & \textbf{0.996} \\ \hline
\end{tabular}
\caption{The performance of multiple methods on out-of-distribution task. Regarding CIFAR10, CIFAR100 and SVHN are considered as near- and far- out-of-distribution datasets, respectively.}
\label{table_ood}
\end{table*}

\begin{table}[!ht]
\centering
\small
\setlength\tabcolsep{3pt}
\begin{tabular}{ccccc}
    \hline
    Method & IMGN-C & IMGN-R & LSUN-C & LSUN-R  \\ \hline
    Softmax (ICLR'16)& 0.639 & 0.653 & 0.642 & 0.647  \\ 
    CROSR (CVPR'19)& 0.721 & 0.735 & 0.720 & 0.749  \\ 
    C2AE (CVPR'19)& 0.837 & 0.826 & 0.783 & 0.801  \\ 
    GFROSR (CVPR'20)& 0.757 & 0.792 & 0.751 & 0.805  \\ 
    CGDL (CVPR'20)& 0.840 & 0.832 & 0.806 & 0.812  \\ 
    PROSER (CVPR'21)& 0.849 & 0.824 & 0.867 & 0.856  \\ 
    ConOSR (AAAI'23)& 0.891 & 0.843 & 0.912 & 0.881  \\ \hline
    MEDAF & \textbf{0.915} & \textbf{0.900} & \textbf{0.922} & \textbf{0.926}  \\ \hline
\end{tabular}
\caption{The macro-F1 results on the CIFAR-10 with various unknown datasets.}
\label{table_f1}
\end{table}
\subsection{Unknown Detection}
In this part, the model requires identifying samples from classes not learned during training on each dataset. Following the setting of \cite{moon2022difficulty}, we conducted the experiments on five image datasets, including CIFAR10 \cite{krizhevsky09learning}, CIFAR+10, CIFAR+50, SVHN \cite{netzer2011reading}, and Tiny-ImageNet \cite{pouransari2014tiny}. The area under the receiver operating characteristic (AUROC) is a threshold-independent metric to measure the model's ability to distinguish knowns and unknowns.

\textbf{MEDAF has a clear advantage on challenging tasks with more unknown samples and high intra-similarity of classes.} 
To avoid unfair comparisons arising from different splits, we adopted unified split information with~\cite{moon2022difficulty,neal2018open} and can be found in the supplementary. Results in Table~\ref{table_auroc} demonstrate that MEDAF outperforms existing discriminative and generative methods on all datasets. It is worth noting that a significant improvement of 9.5\% is obtained on TinyImageNet compared with the recent generative method, which has high variability in object appearance, and some classes in it visually resemble others.

\subsection{Out-of-Distribution Detection}
For measuring the performance of MEDAF in identifying samples that deviate from the learned in-distribution, we conducted OOD detection experiments. Following the setup proposed in \cite{chen2020learning}, we employed CIFAR-100 and SVHN as near- and far-OOD datasets, with CIFAR-10 as the in-distribution dataset. We used the AUROC, DTACC, and AUPR to evaluate the performance. DTACC refers to the highest probability of classification across various thresholds. AUPR measures the performance in a data-imbalanced scenario, and AUIN and AUOUT denote the AUPR value when in- or out-of-distribution samples are positive.

\textbf{MEDAF can well handle near-OOD samples.} Under the settings of near-OOD and far-OOD, the model requires to distinguish between samples with slight or significant differences with in-distribution data. Therefore, the former has stronger requirements for the discriminative ability of model and higher difficulty. Results in Table~\ref{table_ood} show that MEDAF achieves comparable performance with SOTA methods. It is worth noting that MEDAF exhibits superior capabilities when dealing with near-OOD task, demonstrating it can effectively reject samples similar to in-distribution samples, thereby having greater practical significance.

\subsection{Open Set Recognition}
To evaluate the model's ability to classify on closed-set and detect samples in the open space, we conducted a $K+1$ class classification experiment.

Following the protocol introduced by \cite{yoshihashi2019classification}, we used the entire CIFAR10 as the knowns and use processed images derived from LSUN \cite{yu2015lsun} or ImageNet \cite{deng2009imagenet} as unknown samples. We adopted the macro-averaged F1-score as the metric, calculated by averaging the F1-scores of $10+1$ classes.

\textbf{MEDAF can accurately identify both known and unknown samples.} Results in Table~\ref{table_f1} demonstrate that, compared to recently proposed methods, we improve the macro-F1 up to 6.7\%. When faced with different unknown samples, MEDAF has stable $K+1$ classification performance, thus validating the effectiveness of the design in enhancing the distinguishing capability.

\begin{figure}[t]
    \begin{center}
    \includegraphics[width=0.9\linewidth]{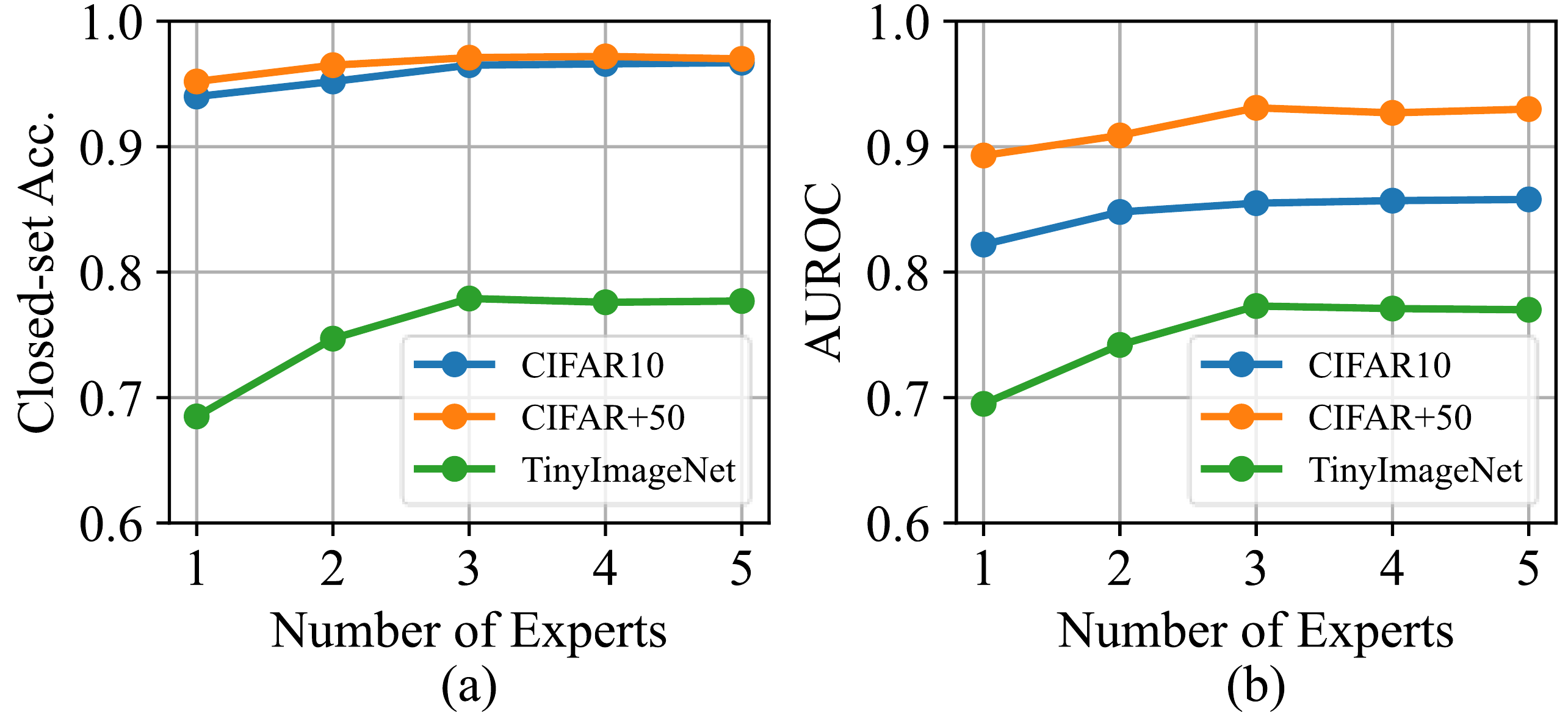}
    \caption{Performance with different expert numbers on multiple datasets, with (a) recording closed-set accuracy and (b) recording AUROC.}
    \label{fig:Fig_expert}
    \end{center}
\end{figure}

\begin{table}[t]
\centering
\small
\setlength\tabcolsep{3.5pt}
\begin{tabular}{ccccccc}
\hline
$\mathcal{L}_d$ & Gating & Acc. &IMGN-R &IMGN-F &LSUN-R &LSUN-F \\ \hline
- & - &  0.936 & 0.944 & 0.915 & 0.952 & 0.901 \\ 
\checkmark & - &  0.950 & 0.977 & 0.951 & 0.987 & 0.946 \\
\checkmark & \checkmark &  \textbf{0.954} & \textbf{0.983} & \textbf{0.953} & \textbf{0.990}  & \textbf{0.950} \\ \hline
\end{tabular}
\caption{The ablation results on loss term and gating network. From top to down, each row represents the results of using a single expert's prediction, adding diverse attention loss, and using average prediction, using weights generated by the gating network.}
\label{table_attention}
\end{table}

\subsection{Ablation Study}
\subsubsection{Fusion on a few experts can achieve outstanding performance.} As shown in Figure~\ref{fig:Fig_expert}, we tested the closed-set classification accuracy and AUROC with different expert numbers. We found that the performance is initially better with more experts but then stabilizes or decreases, which indicates that expanding the model structure has diminishing returns, given the ability to extract comprehensive features.

\subsubsection{Diverse representations and adaptive weighting are both beneficial for OSR.} We conducted ablation experiments on the gating network and diversity loss, which helped the model explore and fuse diverse features. We used CIFAR10 as the known dataset, while variations of LSUN and ImageNet as the unknown datasets. Results in Table~\ref{table_attention} demonstrate that the gating network improves the performance of the model compared to using prediction from a single expert or a simple combination. Moreover, the diversity loss effectively enhances the model’s ability in both known class classification and unknown class detection compared to the baseline.
\begin{table}[t]
\small
\centering
\begin{tabular}{cccc}
\hline
Method & AUROC & TNR@TPR95 & DTACC \\ \hline 
Softmax & 0.797 & 0.448 & 0.735 \\ 
GCPL & 0.823 & - & - \\ 
ARPL & 0.836 & 0.486 & 0.782 \\  \hline
MEDAF & \textbf{0.928} & \textbf{0.583} & \textbf{0.867} \\\hline
\end{tabular}
\caption{Comparison on ImageNet-1k on different metrics.}
\label{table_imgn1k}
\end{table}

\begin{figure}[t]
  \begin{center}
    \includegraphics[width=0.75\linewidth]{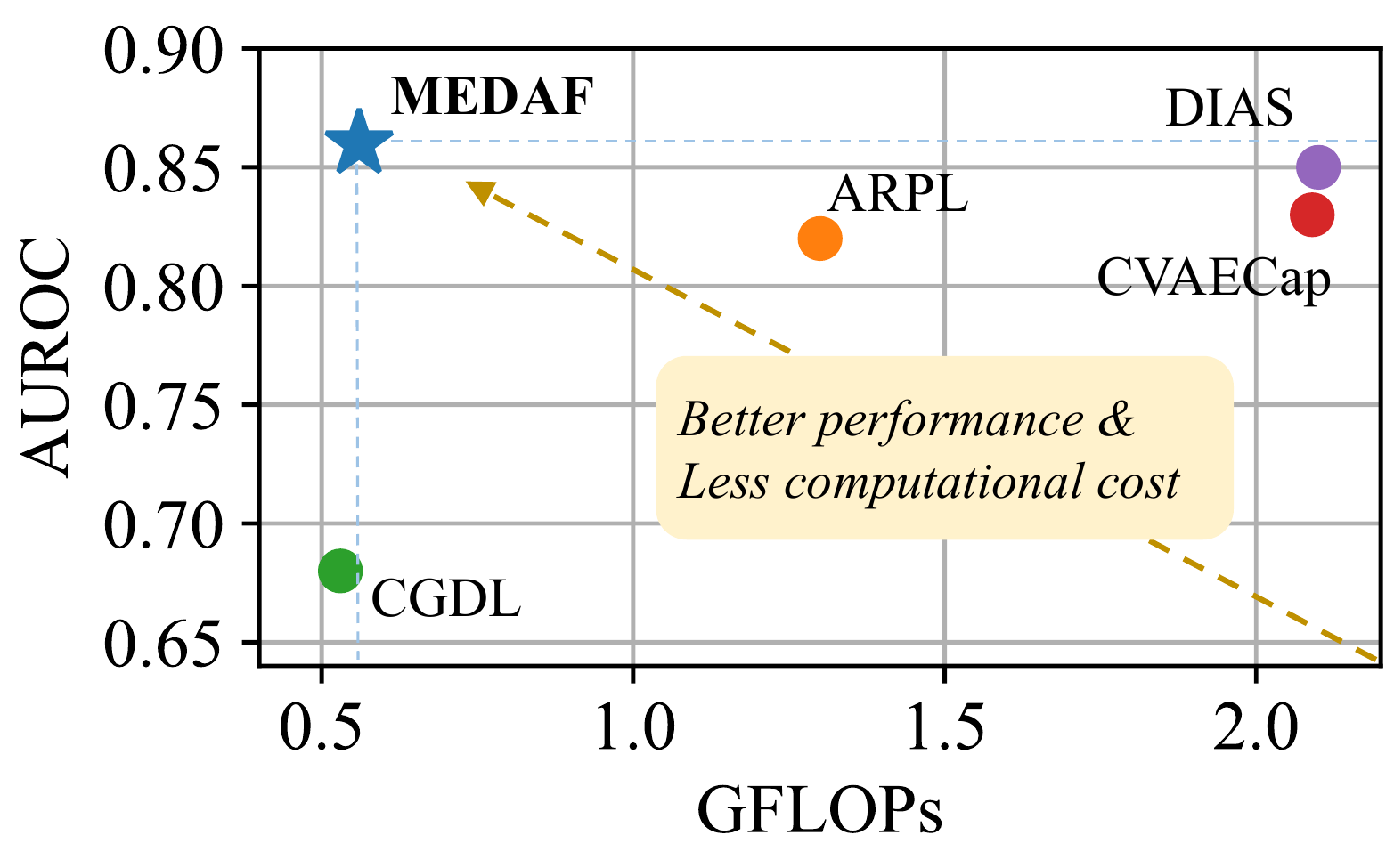}
    \caption{OSR performance against computational cost.}
    \label{fig:Fig_gflops}
    \end{center}
\end{figure}

\begin{table}[t]
\centering
\small
\begin{tabular}{cccc}
\hline
Method & CIFAR10 & CIFAR100 & Tiny-ImageNet \\ \hline 
Plain CNN & 0.940 & 0.716 & 0.637 \\ 
ARPL & 0.941 & 0.721 & 0.657  \\ 
ConOSR & 0.946 & 0.730 & 0.661 \\  \hline
MEDAF & \textbf{0.954} & \textbf{0.770} & \textbf{0.706}\\ \hline
\end{tabular}
\caption{Closed-set classification performance comparison.}
\label{table_close_acc}
\end{table}

\subsection{Further Analysis}
\subsubsection{Large-Scale Benchmark.}
To evaluate the scalability of different methods, we conducted experiments on ImageNet 1k~\cite{russakovsky2015imagenet}. We selected the first 100 classes as known classes, and the remaining 900 are unknown classes. Experimental results are recorded in Table~\ref{table_imgn1k}. MEDAF operates effectively even in challenging scenarios where more unknown samples appear. Furthermore, with these classes requiring detailed differentiation, MEDAF still has outstanding closed-set accuracy.

\subsubsection{Efficiency.}
We compared the computational cost of MEDAF and other methods on CIFAR10. Figure~\ref{fig:Fig_gflops} clearly shows that MEDAF achieves the best performance with little cost. By contrast, most generative models have high GFLOPs, which verifies that they are computationally expensive and difficult to scale to large-scale tasks.

\subsubsection{Closed-Set Accuracy.}
Following the setting in \cite{xu2023contrastive}, we trained the models on CIFAR10, CIFAR100, and the first 100 classes of TinyImagenet to test their closed-set accuracy. Table~\ref{table_close_acc} shows that MEDAF significantly outperforms other methods, especially when the number of classes and task difficulty increases, demonstrating great potential in both closed and open scenarios.

\begin{figure}[t]
  \begin{center}
    \includegraphics[width=0.7\linewidth]{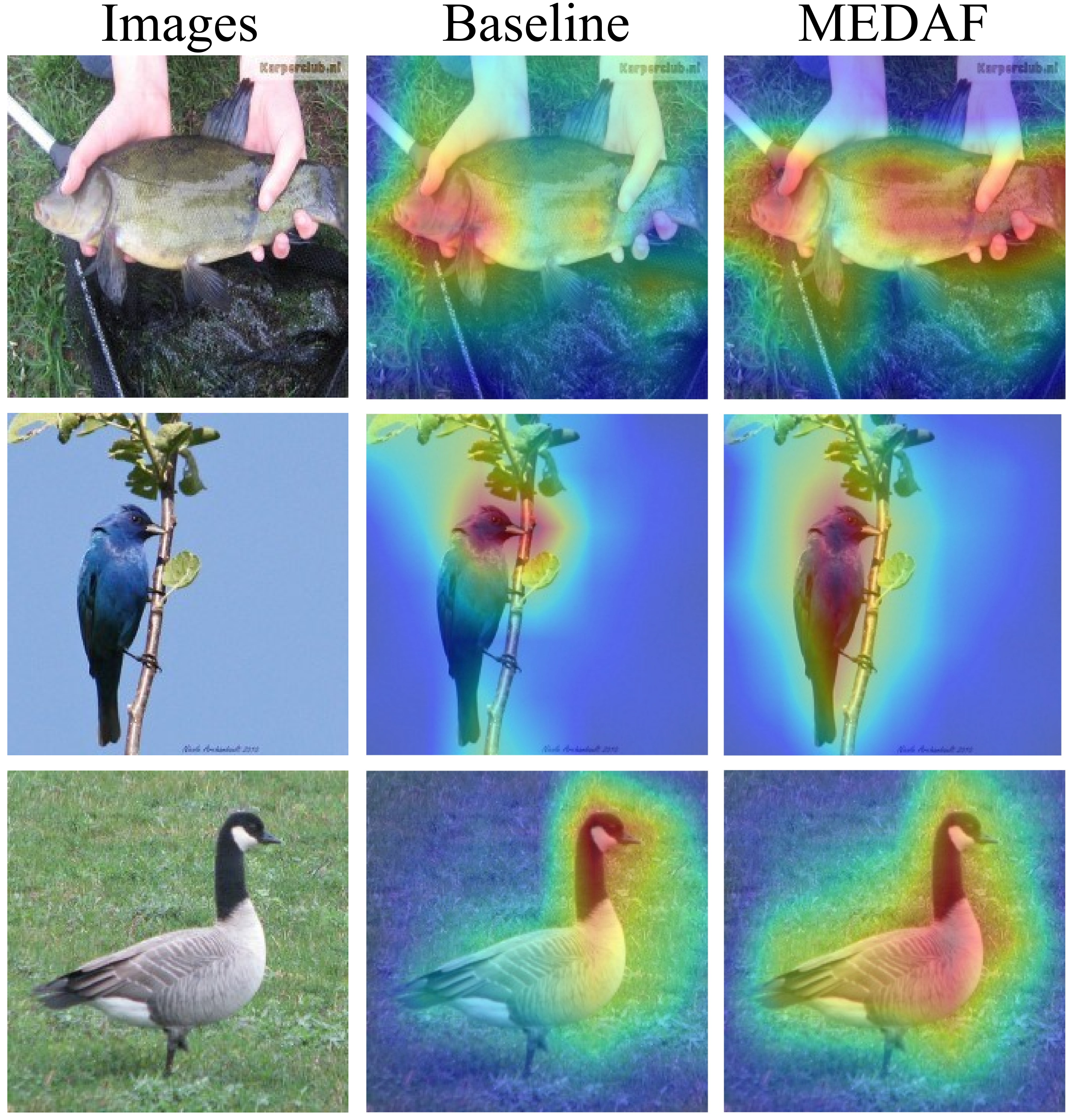}
    \caption{Visulizations on CAMs of baseline and MEDAF.}
    \label{fig:Fig_cam}
    \end{center}
\end{figure}
\subsubsection{Visualizations.}
Figure~\ref{fig:Fig_cam} shows attention visualizations of the baseline method (plain classification model) and the proposed MEDAF. It can be observed that the attention regions are notably more precise and comprehensive when compared to those of a plain classification model. MEDAF learns more diverse features, yielding effective unknown detection as well as robust known recognition.

\section{Conclusion}
In this paper, we analyzed the challenges OSR tasks and found that propose a novel method called Multi-Expert Diverse Attention Fusion (MEDAF). The key insight is that diverse representations that contain supplementary representations in addition to basic ones learned by baselines can effectively reduce the open space risk in OSR. To achieve the goal, MEDAF uses an architecture of multiple experts that share shallow layers while having expert-independent layers. An attention diversity regularization loss is proposed to ensure the learned attention map of each expert is mutually different. By adaptively fusing the logits learned by each expert with a feature-based score function, MEDAF can accurately identify and reject unknown samples. Experiments on both standard and large-scale benchmarks show that MEDAF significantly outperforms other methods with little computational cost. 

\bibliography{ref}
\end{document}